\def\year{2020}\relax
\documentclass[letterpaper]{article} 
\usepackage{subfigure}
\usepackage{algorithm}
\usepackage{amsmath,amsthm}
\usepackage{threeparttable}
\usepackage{booktabs,multirow}
\usepackage[noend]{algorithmic}

\usepackage{aaai20}  
\usepackage{times}  
\usepackage{helvet} 
\usepackage{courier}  
\usepackage[hyphens]{url}  
\usepackage{graphicx} 
\usepackage{graphicx}  
\newcommand{\nop}[1]{}
\newtheorem{definition}{Definition}[section]

\newtheorem{lemma}{Lemma}[section]
\newtheorem{problem statement}{\bf Problem Statement}
\newtheorem{example}{Example}

\title{Question Answering over Knowledge Graphs via Structural Query Patterns}
\author{%
{{Weiguo Zheng${^1}$}, Mei Zhang{${^2}$}}%
\vspace{1.6mm}\\
\fontsize{10}{10}\selectfont\itshape $~^{1}$Fudan University, China;\\ \fontsize{10}{10}\selectfont\itshape $~^{2}$Wuhan University of Science and Technology, China\\
\vspace{0.05in}
 \fontsize{10}{10}\selectfont\ttfamily\upshape \hspace{-0.2in}  zhengweiguo@fudan.edu.cn, zhangmeiontoweb@gmail.com
\\}

\begin{document}

\maketitle

\begin{abstract}
Natural language question answering over knowledge graphs is an important and interesting task as it enables common users to gain accurate answers in an easy and intuitive manner. However, it remains a challenge to bridge the gap between unstructured questions and structured knowledge graphs. To address the problem, a natural discipline is building a structured query to represent the input question. Searching the structured query over the knowledge graph can produce answers to the question. Distinct from the existing methods that are based on semantic parsing or templates, we propose an effective approach powered by a novel notion, structural query pattern, in this paper. Given an input question, we first generate its query sketch that is compatible with the underlying structure of the knowledge graph. Then, we complete the query graph by labeling the nodes and edges under the guidance of the structural query pattern. Finally, answers can be retrieved by executing the constructed query graph over the knowledge graph. 
Evaluations on three question-answering benchmarks show that our proposed approach outperforms state-of-the-art methods significantly.
\end{abstract}

\section{Introduction}\label{sec:introduction}

Querying knowledge graphs like DBpedia, Freebase, and Yago  through natural language questions has received increasing attentions these years.
 In order to bridge the gap between unstructured questions and the structured knowledge graph $G$, a widely used discipline is building a structured query graph $q$ to represent the input question such that $q$ can be executed on $G$ to retrieve answers to the question \cite{DBLP:conf/emnlp/BerantCFL13,DBLP:conf/sigmod/ZhengZLYSZ15,DBLP:journals/tkde/Hu0YWZ18}.
To the end, there are two streams of researches, i.e., semantic parsing based methods and template based methods, both of which suffer from several problems as discussed next.


\noindent \textbf{\textit{Semantic parsing based methods}.}
The aim of semantic parsing is translating the natural language utterances into machine-executable logical forms or programs \cite{DBLP:conf/acl/GardnerDISZ18}. For example, 
the phrase ``director of Philadelphia'' may be parsed as $\lambda x.{\textit{Director}}(x)$ $\wedge$ $\textit{DirectedBy}(\textit{Philadelphia}(\textit{film}), x)$, where $\textit{Director}$, $\textit{DirectedBy}$, and $\textit{Philadelphia}(\textit{film})$ are grounded predicates and entities in the specific knowledge graph.
 Traditional semantic parsers \cite{DBLP:conf/uai/ZettlemoyerC05,DBLP:conf/acl/WongM07,DBLP:conf/emnlp/KwiatkowksiZGS10} require a lot of annotated training examples in the form of syntactic structures or logical forms, which is especially expensive to collect for large-scale knowledge graphs. Another problem is the mismatch between the generated logic forms and structures (including entities and predicates) that are specified in knowledge graphs \cite{DBLP:conf/emnlp/KwiatkowskiCAZ13,DBLP:conf/acl/BerantL14,DBLP:journals/tacl/ReddyLS14}.
 In order to solve the problems above,
 several efforts have been devoted to lifting these limitations \cite{DBLP:conf/acl/YihCHG15,DBLP:conf/coling/BaoDYZZ16}. They leverage the knowledge graph in an early stage by applying deep convolutional neural network models to match questions and predicate sequences. It is required to identify the topic entity $e$ and a core inferential chain that is a directed path from $e$ to the answer. Then the final executable query is iteratively constructed based on the detected chain. However, it is hard to pick out the correct inferential chains (35\% of the errors are caused by the incorrect inferential chains in STAGG \cite{DBLP:conf/acl/YihCHG15}). Moreover, it is unreasonable to restrain the chain as a directed path since it may be a general path regardless of the direction in many cases. For instance, Figure~\ref{fig:query graphs} presents the query graphs for two questions ``\textit{$q_1$: Who is starring in Spanish movies produced by Benicio del Toro}?'' and ``\textit{$q_2$: Which artists were born on the same date as Rachel Stevens?}'' targeting DBpedia, neither of which contains directed paths. In addition, the search space is uncertain and difficult to determine when to terminate.


 \begin{figure}[h]
  \centering
  \subfigure[Query graph for $q_1$]{
    \label{fig:affected subgraph Ga} 
    \includegraphics[scale=0.55 ]{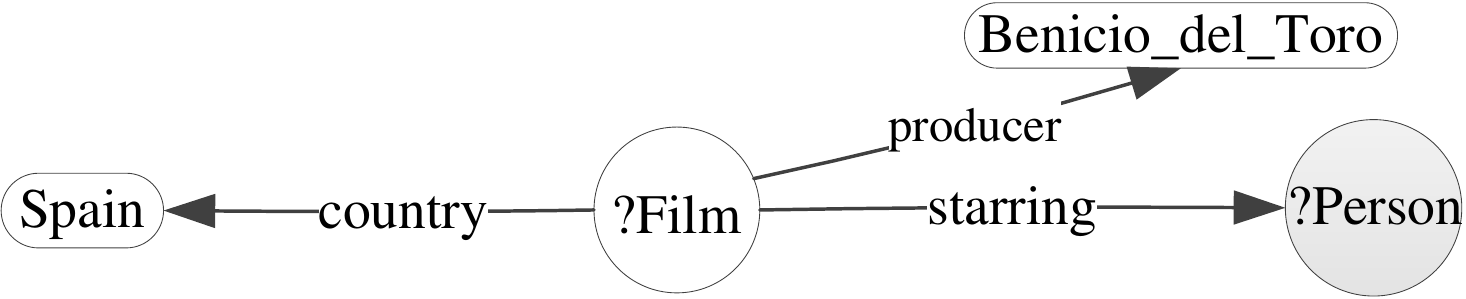}}
  \subfigure[Query graph for $q_2$]{
    \label{fig:affected subgraph Gr} 
    \includegraphics[scale=0.55 ]{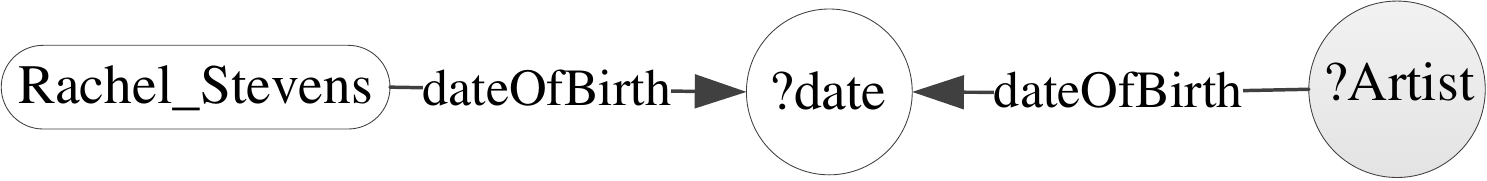}}
  \caption{Query graphs for questions $q_1$ and $q_2$}
  \label{fig:query graphs} 
\end{figure}

Instead of training semantic parsers, several methods that are built upon the existing dependency parses have been proposed \cite{DBLP:conf/sigmod/ZouHWYHZ14,DBLP:conf/clef/RusetiMRT15}. They try to generate the query graphs according to the dependency parsing results and pre-defined rules. Clearly, it is extremely difficult to enumerate all the rules and eliminate conflicts among the rules.

\noindent \textbf{\textit{Template-based methods.}}
A bunch of researches focus on using templates to construct query graphs \cite{DBLP:conf/www/UngerBLNGC12,DBLP:conf/ijcai/CuiXW16,DBLP:conf/www/AbujabalYRW17,DBLP:journals/pvldb/ZhengYZC18}, where a template consists of two parts: the natural language pattern and SPARQL query pattern. The two kinds of patterns are linked through the mappings between their slots. In the offline phase, the templates are manually or automatically constructed. In the online phase, it tries to retrieve the template that maps the input question. Then the template is instantiated by filling the slots with entities identified from the question. The generated query graph is likely to be correct if the truly matched template is picked out. Nevertheless, the coverage of the templates may be limited due to the variability of natural language and a large number of triples in a knowledge graph, which will lead to the problem that many questions cannot be answered correctly.  Furthermore, automatically constructing and managing large-scale high-quality templates for complex questions remain open problems.

\noindent \textbf{Our Approach and Contributions.}
As discussed above, the semantic parsing based algorithms show good scalability to answer more questions, while the template-based methods exhibit advantage in precision. Hence, it is desired to design an effective approach that can integrate both of the two strengths.
%
To this end, there are at least two challenges to be addressed.

\noindent \emph{Challenge 1. Devising an appropriate representation that can capture the query intention and is easy to ground to the underlying knowledge graph.} The representation is required to intuitively match or reconstruct the query intention of the input question. Meanwhile, it is natural to be grounded to the knowledge graph, which is beneficial to improve the precision of the system.

\noindent \emph{Challenge 2. Completeness of representations should be as high as possible.} Although the template-based methods perform good in terms of precision, they suffer the problem of template deficiency in real scenarios. Guaranteeing the completeness of representations is crucial to enhance the processing capacity. Moreover, in order to reduce the cost of building such a question answering system, the representations should be easy to construct.

Rather than using the semantic parsing or templates, \emph{we propose a novel framework based on structural query patterns to build a query graph for the input question in this paper}. It comprises of three stages, i.e., structural query pattern (shorted by \textit{SQP}) generation, \textit{SQP}-guided query construction, and constraint augmentation.
In principle, instead of parsing the question $q$ into a logic form that is equipped with specific semantic arguments (including entities and predicates), we just need to identify the shape or sketch of $q$'s query graph in the first stage. It benefits from two folds:
(1) The number of structural patterns for most questions is limited so that they can be enumerated in advance. For instance, there are 4 structural patterns for the questions in LC-QuAD \cite{DBLP:conf/semweb/TrivediMDL17} that is a benchmark for complex question answering over DBpedia.
(2) It is easy to produce a structured pattern with high precision compared to generating complicated logic forms. In the second stage, we build the query graph by extending one entity that is identified according to the question $q$. The construction proceeds under the guidance of the structural pattern. Hence, the search space can be reduced rather than examining all the predicates adjacent to an entity. Furthermore, it is straightforward to determine whether the extension procedure can terminate or not. 
Finally, the constraints specified in the question $q$ are detected to produce the complete structured query for $q$.
Note that the procedure involves multiple steps, i.e., \textit{SQP} generation, entity linking, and relation selection. 
In summary, we make the following contributions in this paper:
\begin{itemize}
    \item We propose a novel framework based on structural query patterns to answer questions over knowledge graphs;
    \item We present an approach that generates a query graph for the input question by applying structural query patterns;
    \item Experimental results on two benchmarks show that our approach outperforms state-of-the-art algorihtms.
\end{itemize}

\section{Preliminaries}\label{sec:background}

\begin{figure*}[t]
\begin{center}
    \includegraphics[scale=0.32]{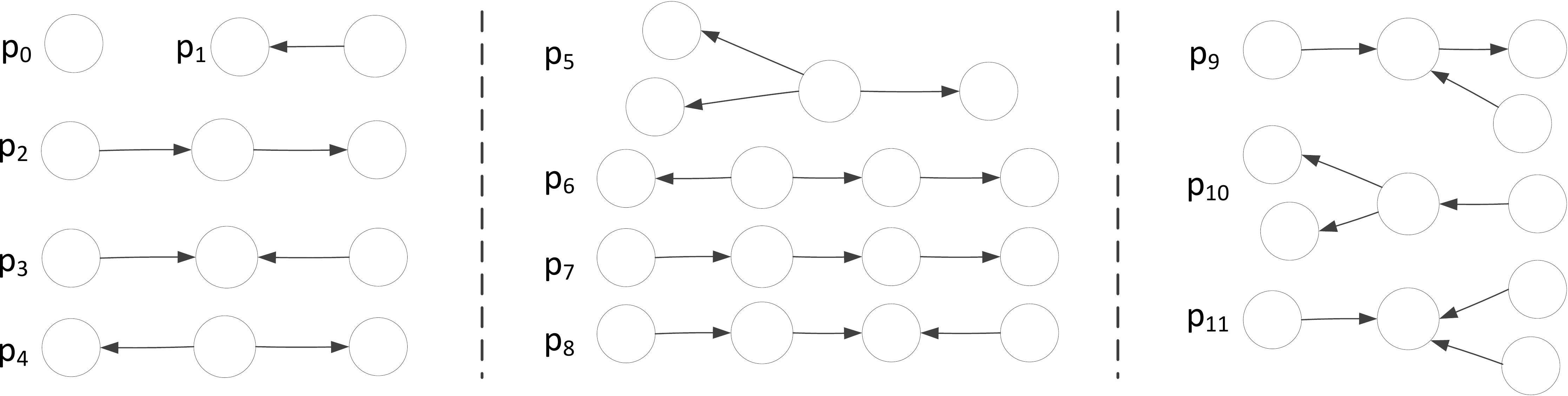}
   \caption{\small Structural query patterns.}
   \label{fig:structural query patterns}
   \vspace{-0.1in}
\end{center}
\end{figure*}



We aim to build a query graph $g$ for the question such that $g$ can be executed on the knowledge graph, where the {\emph{query graph}} $g$ is a graph representation of the input question and can be executed on $G$. Formally, it is defined as Definition~\ref{def:query graph}.

\begin{definition}\label{def:query graph}
(\emph{\textbf{Query graph, denoted by $g$}}). A query graph for the input question satisfies the following conditions: (1) each node in $g$ corresponds to an entity $e \in G$, a value $v \in G$, or a variable; (2) each edge in $g$ corresponds to a predicate or property in $G$; (3) $g$ is subgraph isomorphic to $G$ by letting the variable in $g$ match any node in $G$.
\end{definition}




The existing semantic parsing based methods map a natural language question $q$ to logical forms that contain phrases from $q$ or entities/predicates in the knowledge graph. The template-based algorithms bridge the gap between unstructured $q$ and structured $G$ by using templates. Different from them, we propose a novel approach by applying structural query patterns.


\begin{definition}
(\textbf{\emph{Structural query pattern, shorted as {\textit{SQP}}}}). The structural query pattern is the structural representation of the question $q$, that is the remaining structure obtained by removing all the node and edge labels from the query graph $g$ of $q$.
\end{definition}

Note that the type-node and its adjacent edges are removed as well. We can call a structural query pattern as a query sketch or a shape of the query graph.

Since most questions involve just a few entities in the knowledge graph $G$, the number of structural query patterns is very limited in the real scenario. Therefore, we can enumerate all these patterns in advance. We observe that each structural query pattern of all the questions in LC-QuAD \cite{DBLP:conf/semweb/TrivediMDL17}, QALD-7 \cite{DBLP:conf/esws/UsbeckNHKRN17}, QALD-8 \cite{DBLP:conf/semweb/UsbeckNCRN18}, and QALD-9 \cite{DBLP:conf/semweb/UsbeckGN018} contains 4 nodes at most. Note that LC-QuAD consists of 5,000 questions, 78\% of which are complex questions with multiple relations and entities. The structural query patterns consisting of 4 nodes at most are listed in Figure~\ref{fig:structural query patterns}. There are 12 structural query patterns in total, where the first pattern $p_0$ contains only one node as the type nodes are removed from the query graph as well. For instance, the structural query pattern for the question ``give me all the automobiles'' is $p_0$.
   The questions in LC-QuAD involve only structural query patterns $p_1$-$p_4$. The questions in QALD-7 and QALD-8 involve structural query patterns $p_0$-$p_6$ and $p_0$-$p_5$, respectively. The questions in QALD-9 involve structural query patterns $p_0$-$p_4$ and $p_7$. In comparison, to capture these questions precisely, it requires thousands of templates. Clearly, the structural query patterns exhibit a strong ability to capture the structural representations of natural language questions.

\section{Structural Query Pattern Recognition}\label{sec:structural pattern recognition}


Given a question $q$ in the online phase, we need to produce the structural query pattern (\textit{SQP}) for $q$. Since the structural query patterns have been listed in advance, we can take the problem of \textit{SQP} recognition as a classification task, where each entry corresponds to a structural query pattern.


\noindent\textbf{Data preparation.}
First of all, we prepare the training data carefully to enhance overall performance. Notice that our defined structural query patterns do not contain any specific labels including phrases, entities, predicates, and properties. However, a natural language question consists of a sequence of words carrying semantic meanings. To make it fit the classification model well, we use the syntax tree of each question rather than the question itself. By utilizing the existing syntactic parsing tools, e.g., Stanford parser \cite{DBLP:conf/emnlp/ChenM14}, we can get the syntactic structure of each question. To avoid the effect of specific words, we remove them from the syntactic parsing results and only retain the syntactic tags and their relations.




 %
 According to the structure of the SPARQL query, each question $q$ is assigned a category label that corresponds to one of the 12 structural query patterns listed in Figure~\ref{fig:structural query patterns}.
   Several question answering datasets are available, such as LC-QuAD and QALD, which provide both questions and the corresponding SPARQL queries. Thus it is easy to collect pairs of syntactic structure and the corresponding category label as the training data.


 \noindent \textbf{Model training.}
  In this phase, we train the classification model to predict the category label for an input question. In this paper, we choose two models Text CNN model \cite{DBLP:conf/emnlp/Kim14} and RNN Attention model \cite{DBLP:conf/naacl/YangYDHSH16} to train the data collected from the previous subsection. The Text CNN model performs well in short text classification. Its basic principle is learning task-specific vectors through fine-tuning to offer further
gains in performance.
   The RNN Attention model imports attention to capture the dependency of long text and can maintain key information from the text. However, these two models can only output a single label.

   The output layers of the two models above just return the label with the largest confidence score.
    In order to increase the ability to deliver the correct label, we modify the two models so that they can assign each label $l$ a confidence score that represents the probability of being the correct label. In the online phase, we use the top-$k$ labels with the highest score.




\noindent \textbf{Model Ensemble.}
Benefiting from the capability on capturing useful information in long text, the RNN Attention model performs better on complex and long questions which contains more than one entity. In contrast, we find that Text CNN model is better than RNN Attention model on dealing with short and simple questions.

\begin{example}
Let us consider the question ``Name the city whose province is Metropolitan City of Venice and has leader as Luigi Brugnaro?''. It is a complex long question as multiple relations and entities are involved. The prediction result of Text CNN model is pattern $p_2$ in Figure~\ref{fig:structural query patterns}. However, the correct pattern should be pattern $p_4$ as delivered by the RNN Attention model.
\end{example}

Since these two models exhibit different advantages when predicting category label, they can be assembled to make a prediction. We use a simple neural network as the training model to ensemble these two models. 
A soft max output layer is included to compute the top-$k$ structural query patterns with the largest score.
 The ranked list of \emph{SQPs}  returned by RNN Attention model and Text CNN model for each question are integrated as the training data. 


\section{Query Graph Generation}\label{sec:query formulation}


 Generally, a query graph contains one entity at least, which can help reduce the search space. Hence, we need to identify the entity $e$ from the knowledge graph, where $e$ corresponds to a phrase (named entity) in the question. It is actually the task of entity linking. Then the query graph is constructed by extending $e$ under the guidance of \emph{SQP}.

\subsection{Entity Linking}\label{subsec:entity linking}



We perform entity linking that finds the entity corresponding to a phrase in the question from the underlying knowledge graph. Conducting entity linking involves two steps, i.e., identify named entities in the question, and then find their matching entities in the target knowledge graph $G$.

 In the first step, we use the named entity recognition (shorted as NER) model \cite{DBLP:conf/naacl/LampleBSKD16} to recognize entity keywords in the given question, where the model is built based on bidirectional LSTMs and conditional random fields. The identified entity keywords are called entity phrases. In the second step, we link the entity phrase to the entity in $G$ by computing the similarity between the entity phrase and candidate entities. 
Note that it is not necessary to identify all the entities in the question as we just need one entity to locate the candidate subgraphs in the knowledge graph.


 We observe that there are two problems to be addressed, i.e., phrase truncation and multiple mapping entities. 
 Algorithm~\ref{alg:entity linking} outlines the procedure.

%
%

\noindent (1) \emph{\textbf{Phrase truncation}}. A phrase may be truncated, which will lead to the false entity phrase and mapping/missing entity. For instance, in the question ``Rashid Behbudov State Song Theatre and Baku Puppet Theatre can be found in which country?'' a truncated phrase ``Song Theatre'', which has no mapping entity in DBpeida, will be identified by the NER model \cite{DBLP:conf/naacl/LampleBSKD16}. In contrast, the correct entity phrase should be ``Rashid Behbudov State Song Theatre'' that is mapped to the entity $\langle$http://dbpedia.org/resource/Rashid\_Behbudov\_State\_Song\_ Theatre$\rangle$. To solve the problem, we extend the entity phrase that is identified by the NER model. As shown in Algorithm~\ref{alg:entity linking}, each time we add the phrase $phr'$ containing $phr$ into the set of extended phrase $PX$ such that $|phr'| \le \theta$, where $|phr|$ is the number of words in $phr$, and $\theta$ is the maximum length of possible phrases. Then the original phrase and its extended phrases constitute the group of phrases $PX$. Since the phrases in $PX$ are extended based on the identical phrase, only one of them can be correct at most. Instead of determining which one phrase in $PX$ is correct directly, we resort to their mapping entities in the knowledge graph as discussed next.

\noindent  (2) \emph{\textbf{Multiple mapping entities}}. Finding the candidate mappings for each entity phrase by using DBpedia Lookup may return multiple entities. Generally, there is only one entity in $G$ that matches an entity phrase in the question. Hence, it is desired to select the correct one. To the end, we compute a matching score for each candidate entity, where the matching score between each entity phrase $phr$ and candidate entity $e$ is computed as shown in Equation~\ref{equ:matching score}.
\begin{equation}\label{equ:matching score}
ms(phr,e) = \alpha_1 \cdot imp(e) + \alpha_2 \cdot sim(phr,e) + \alpha_3 \cdot rel(q, evd(e))
\end{equation}
%

\begin{algorithm}[tb]
\caption{Entity Linking}
\label{alg:entity linking}
\textbf{Input}: Input question $q$ and knowledge graph $G$\\
\textbf{Output}: Entity $e_0$ matching a phrase $phr_0$ in $q$\\
\vspace{-0.18in}
\begin{algorithmic}[1] 
\STATE $EP \leftarrow$ identify all the entity phrases using the NER tool
\STATE $score \leftarrow 0$
\FOR{each phrase $phr$ in $EP$}
        \STATE $PX \leftarrow$ the phrases containing $phr$ whose length is not larger than $\theta$
        \STATE $PX \leftarrow PX \cup \{phr\}$
        \STATE $EC \leftarrow \emptyset$
        \FOR{each phrase $phr'$ in $PX$ }
            \IF{$phr'$ can match at least one entity in $G$}
                \STATE $EC \leftarrow EC \cup$ the candidate entities matching $phr'$
            \ENDIF
        \ENDFOR
        \FOR{each candidate entity $e$ in $EC$}
            \STATE compute the matching score $ms(phr,e)$ between $phr$ and $e$
        \ENDFOR
        \STATE $e' \leftarrow$ the entity with the largest matching score in $EC$
        \IF{$score < ms(phr,e')$}
            \STATE  $score \leftarrow ms(phr,e')$
            \STATE $e_0 \leftarrow e'$,  $phr_0 \leftarrow phr$
        \ENDIF
\ENDFOR

\STATE \textbf{return} $e_0$ and $phr_0$
\end{algorithmic}
\end{algorithm}

As defined above, the matching score consists of three components, i.e., the importance of the entity $e$ (denoted as $imp(e)$), the similarity between $phr$ and $e$ (denoted as $sim(phr,e)$), and the relevance between $q$ and the evidence text of $e$ (denoted as $rel(q,evd(e))$). They are formally defined as shown in Equations~(\ref{equ:imp})-(\ref{equ:relevance sen}). The parameters $\alpha_1$, $\alpha_2$, and $\alpha_3$ are weights of the three components, respectively.

\begin{equation}\label{equ:imp}
imp(e) = \frac{1}{rank(e)}
\end{equation}
where $rank(e)$ denotes the rank of $e$ among all the candidate mappings of $phr$ in terms of their term frequency in text corpus, e.g., Wikipedia documents. The principle is an entity is more important if it appears in more documents.

\begin{equation}\label{equ:sim phr e}
sim(phr,e) = \frac{1}{lev(phr,e)+1}
\end{equation}
where $lev(phr,e)$ can be computed with the widely used metric, Levenshtein distance \cite{DBLP:journals/Levenshtein}, for measuring the difference between two strings.

The relevance between $q$ containing the entity $e$ and the evidence text $evd(e)$ of $e$, e.g., the corresponding Wikipedia page $doc$ of entity $e$ can be taken as its evidence text. Then we can compute the similarity between the question $q$ and each sentence $s_i$ in $doc$ as shown in Equation~(\ref{equ:relevance sen}), where $vec(q)$ and $vec(s_i)$ denote the vector representations of $q$ and $s_i$, respectively. 
\begin{equation}\label{equ:relevance sen}
rel(q,evd(e)) = \arg max_{s_i \in doc} \frac{vec(q) \cdot vec(s_i)}{\| vec(q) \| \  \| vec(s_i) \|}
\end{equation}

Finally, the entity with largest matching score is returned.

\subsection{SQP-guided Query Graph Construction}\label{subsection:SQP-guided query formulation}

With the predicted structural query pattern $p$ and one identified entity $e$, we are ready to construct the query graph. The basic idea is instantiating the pattern $p$ through a data-driven search under the guidance of $p$. Specifically, the search starts from the entity node $e$, and retrieves a subgraph that contains $e$ and is structurally isomorphic to the structural query pattern by ignoring all the node/edge labels.
%
 In order to construct the query graph, two tasks should be completed, i.e., locate the position of $e$ in $p$ and query graph extension. 

\noindent \underline{\emph{Task 1}: \textit{Locate the position  of entity node $e$ in the pattern $p$}}.
Although both $p$ and $e$ can be obtained as discussed above, the position of the node $e$ in $p$ is unknown.
 To locate the position of entity $e$ in the pattern $p$, we introduce an important observation with the ``\emph{non-redundancy assumption}'': if the question $q$ has only one return variable, the words in $q$ are all helpful to depict the query intent. 

\begin{lemma}\label{lemma:intermediate node}
 The entity $e \in G$ identified for the question $q$ is not an intermediate node in the structural query pattern $p$.
\end{lemma}
\begin{proof}
  The underlying rationale is that the question will contain useless words if $e$ is an intermediate node in $p$. The proof can be achieved by contradiction. Let us assume that $e$ is an intermediate node. It will lead to a triple $\langle$$e, r , x$$\rangle$ at least, where $x$ is an entity or a literal string, and $r$ is the incident relation. Clearly, this triple contributes nothing to restraining the variable in other triples as they are all constant nodes. It indicates that the node $x$ and relation $r$ are useless to specify the answers, which contradicts the \emph{non-redundancy assumption} aforementioned.
\end{proof}

Lemma~\ref{lemma:intermediate node} works under the premise that the query graph of a given question does not contain any cycles.
 Note that all the query graphs in the two benchmarks used in the paper are trees. Furthermore, the answers can be retrieved even if the corresponding query graph is not a tree since the tree has fewer constraints than a graph. Then we can refine the answers according to the information in the question that is not covered by the tree pattern.

\begin{example}
 Assume that $p$ is the third pattern, i.e.,  $p = p_2$, in Figure~\ref{fig:structural query patterns} and $e$ is the intermediate node. Thus one of the two leaf nodes represents the return variable, and the other node will be an entity $e'$ or a literal string $l$. It suggests that there is a triple $\langle$$e, r_1 , e'$$\rangle$, $\langle$$e', r_1 , e$$\rangle$, or $\langle$$e, r_1 , l$$\rangle$\footnote{The literal node cannot be a starting node of an edge in knowledge graphs.}, where $r_1$ is the adjacent relation to $e'$ or $l$. As both $e$ and $e'$ (resp.~$l$) are specified entities (resp.~literal string), the three triples will contribute nothing to restraining the variable in the other triple $\langle$$?x, r_2 , e$$\rangle$ or $\langle$$e, r_2 , ?x$$\rangle$, where $r_2$ is the incident relation to the variable node $?x$. Hence, we can conclude that $e$ is not an intermediate node in $p$. The analysis holds for other patterns as well.
\end{example}

 As a supporting proof through real data analysis, we find that all the entities are not intermediate nodes in the benchmarks LC-QuAD and QALD.

\noindent \underline{\textit{Task 2: Query graph extension.}} With the entity and its position in $p$, we build the query graph in this task. The main principle is extending the query graph (initially just an entity node e) gradually by including relations, entities, or variables to $p$ under the help of the structure in $p$. The expanding procedure is depicted in Algorithm~\ref{alg:query graph extension}. Note that we select the pattern with the largest confidence score for simplicity.

\begin{algorithm}[tb]
\caption{Query graph extension}
\label{alg:query graph extension}
\textbf{Input}: Entity $e$, input question $q$, structural query pattern $p$, and knowledge graph $G$\\
\textbf{Output}: Query graph $Q$ for $q$
\begin{algorithmic}[1] 
\STATE $Q \leftarrow p$, \hspace{0.1in} $cn \leftarrow e$ 
\STATE $NS \leftarrow$ non-intermediate nodes in $p$
\WHILE{$Q$ has unlabeled edges or nodes}
    \STATE $R \leftarrow \emptyset$.
    \IF{the nodes $NS$ in $p$ have outgoing neighbors}\label{line:alg-line1}
        \STATE $R \leftarrow$  $R$ $\cup$ outgoing relations of $cn$.
    \ENDIF
    \IF{the nodes $NS$ in $p$ have incoming neighbors}
        \STATE $R \leftarrow$  $R$ $\cup$ incoming relations of $cn$.
    \ENDIF \label{line:alg-line2}
    \STATE $r \leftarrow$ the relation in $R$ that is the most relevant to $q$. \label{line:alg-line3}
    \IF{$cn$ corresponds to an entity phrase in $q$}
        \STATE assemble $cn$ and $r$ into $Q$.
    \ELSE
        \STATE assemble variable and $r$ into $Q$.
    \ENDIF

    \STATE $NS \leftarrow$ the unexplored node adjacent to explored structure $Q_L$.
    \STATE $cn \leftarrow$ the entities corresponding to $NS$ that are adjacent to $Q_L$ in $G$.
\ENDWHILE

\STATE \textbf{return} $Q$
\end{algorithmic}
\end{algorithm}

Since a structural query pattern may contain multiple non-intermediate nodes, it is not trivial to determine the correct one. We propose an extension procedure following a data-driven manner. Initially, we make a copy $Q$ of the structural query pattern $p$. If the non-intermediate nodes in $p$ have both incoming edges and outgoing edges, e.g., \textit{SQP}s 1, 2, 6, 7, 9 and 10 in Figure~\ref{fig:structural query patterns}, the incoming relations and outgoing relations of entity $e$ will be collected. Otherwise, we just need to consider the incoming or outgoing relations (lines~\ref{line:alg-line1}-\ref{line:alg-line2} in Algorithm~\ref{alg:query graph extension}). Then we compute the relevance between each candidate relation $r$ and $q$. The relation with the largest relevance is selected. As a relation $r$ may be composed of multiple words, e.g., dateOfBirth, it should be split to get a sequence of words. Then the relevance $rel(q,r)$ between $q$ and each candidate relation $r$ can be calculated by Equation~(\ref{equ:relevance}), where $\lambda$ is a weight ranging from 0 to 1, $q_i$ and $r_j$ represents the $i$th and $j$th words in $q$ and $r$, respectively.
%
\begin{equation}\label{equ:relevance}
\resizebox{.9\linewidth}{!}{$
    \displaystyle
     rel(q,r) = \sum_{i=1}^{|q|} \sum_{j=1}^{|r|} \lambda \cdot cos(q_i,r_j) + (1-\lambda) \cdot \frac{1}{lev(q_i,r_j)+1}
$}
\end{equation}%

\begin{table*}[tp]
 \centering
\begin{small}
  \caption{The performance when only the query graphs are generated.}  \label{tab:query graph}
  \begin{threeparttable}
    \begin{tabular}{lccccccccc}
    \toprule
    \multirow{2}{*}{Method}&
    \multicolumn{3}{c}{ LC-QuAD}&\multicolumn{3}{c}{ QALD-8}&\multicolumn{3}{c}{ QALD-9}\cr
    \cmidrule(lr){2-4} \cmidrule(lr){5-7}\cmidrule(lr){8-10}
    &Precision&Recall&F1-Measure& \hspace{0.1in} Precision&Recall&F1-Measure &Precision&Recall&F1-Measure\cr
    \midrule
    Frankenstein & 0.480  & 0.490 & 0.485 & - & - & - & - & - & - \cr
    qaSearch & 0.357  & 0.336 & 0.344 & 0.243 & 0.243 & 0.243  & 0.198 & 0.191 & 0.193 \cr
    \textit{qaSQP}    & 0.748  & 0.704 & 0.718 & 0.439 & 0.439 & 0.439  & 0.401 & 0.413 & 0.405\cr
    \textit{qaSQP-CE}    & {0.835}  & {0.813} & {0.827} & {0.558} & {0.663} & {0.620}  & 0.522 & 0.625 & 0.568  \cr
    \bottomrule
    \end{tabular}
    \end{threeparttable}
  \end{small}
\end{table*}

As shown in Equation~(\ref{equ:relevance}), we use two metrics to measure the relevance between two words $w_1$ and $w_2$. The first one is the cosine score $cos(w_1, w_2)$ between two vectors of $w_1$ and $w_2$ which are obtained by training the glove data with the model of word2vec. Two words are semantically closer to each other if their cosine score is larger. The other metric is Levenshtein distance $lev(w_1, w_2)$ that calculates the edit cost between two words.
After obtaining the relation $r$ that is the most relevant to $q$ as shown in line~\ref{line:alg-line3} of Algorithm~\ref{alg:query graph extension}. The specific position of $e$ can be determined according to the direction of $r$. Then we can include $e$ and $r$ into $Q$. Taking the node and entities adjacent to the subgraph $Q_L$ (that has been labeled with entities, variables, and relations currently) as a starting node, the extension procedure proceeds iteratively until all the nodes and edges have been labeled. Finally, the query graph $Q$ is returned.

\subsection{Constraint Augmentation}

Actually, executing the query graph above can return a list of answers which contain the correct one. However, it may generate undesired entities or values as a question may put some constraints on the query graph. For instance, the question ``\textit{What is the highest mountain in Italy?}'' specifies the ordinal constraint ``highest'' on mountains.

We divide the constraints into 4 categories as follows:
\begin{itemize}
  \item answer-type constraint, e.g., ``which actor'';
  \item ordinal constraint, e.g., ``highest'';
  \item aggregation constraint, e.g., ``how many'';
  \item comparative constraint, e.g., ``larger than''.
\end{itemize}

Similar to the approaches \cite{DBLP:conf/acl/YihCHG15,DBLP:conf/coling/BaoDYZZ16}, we employ simple rules to detect these constraints and augment them to the query graph.

\section{Experiments}\label{sec:experiments}
In this section, we evaluate the proposed method systematically and compare it with the existing algorithms.

\subsection{Datasets and Experimental Settings}\label{subsec:datasets}
We use DBpedia \cite{DBLP:conf/semweb/AuerBKLCI07} as the target knowledge graph. DBpedia is an open-domain knowledge graph that consists of 6 million entities and 1.3 billion triples as reported in the statistics of DBpedia 2016-10 .

Two question answering benchmarks LC-QuAD \cite{DBLP:conf/semweb/TrivediMDL17} and QALD \cite{DBLP:conf/esws/UsbeckNHKRN17} delivered over DBpedia are used to evaluate our proposed approach.

{\small
\begin{itemize}
    \item LC-QuAD is a gold standard question answering dataset that contains 5000 pairs of natural language questions and
SPARQL queries, 728 of which are simple questions with single relation and single entity.
    \item QALD-8 \cite{DBLP:conf/semweb/UsbeckNCRN18} and QALD-9 \cite{DBLP:conf/semweb/UsbeckGN018}. QALD is a long-running question-answering evaluation campaign. It provides a set of natural language questions, the corresponding SPARQL queries and answers. QALD-8 contains 219 training questions and 42 test questions. QALD-9 contains 408 training questions and 150 test questions. 
\end{itemize}
}


We randomly select 500 questions from LC-QuAD as the test data.
Our models are trained for 100 epochs, with early stopping enabled based on validation accuracy. We use the 80-20 split as train and validation data.

  In RNN Attention model, we set the dimensionality of character embedding to 128 and dropout keep probability to 0.5. The number of hidden units is 128. The number of attention units is 128. The number of hidden size is 1.
In Text CNN model, we set dimensionality of character embedding to 128. The number of filters per filter is 128. The dropout keep probability is 0.5. L2 regularization lambda is 0.001.


Following conventions as shown in gAnswer2 \cite{DBLP:journals/tkde/Hu0YWZ18},
 the macro precision, recall, and F1-measure are used to evaluate the performance.
 We compare our method, denoted by \emph{\textbf{qaSQP}}, with 
Frankenstein \cite{DBLP:conf/www/SinghRBSLUVKP0V18}, QAKIS \cite{DBLP:conf/esws/CabrioCGH13}, QASystem \cite{DBLP:conf/semweb/UsbeckGN018}, TeBaQA \cite{DBLP:conf/semweb/UsbeckGN018}, WDAqua \cite{DBLP:conf/esws/UsbeckNHKRN17}, gAnswer2 \cite{DBLP:journals/tkde/Hu0YWZ18}, and qaSearch, where qaSearch constructs the query graph following a data-driven search rather than using the \emph{SQP} as guidance

\subsection{Experimental Results}\label{subsec:results}

\noindent  \textbf{\emph{Comparing with the previous methods}}.
Table~\ref{tab:query graph} presents the performance in terms of generated query graphs on the three datasets, where 
\textit{qaSQP-CE} represents the proposed method that is fed by one correctly identified entity initially. As can be seen from the table, our proposed method outperforms the existing approach by a large margin 17.4\% absolute gain on LC-QuAD. The performance improves further if our system is given one correctly identified entity in DBpedia. The performance on QALD-8 and QALD-9 get worse than that on LC-QuAD. There are two main reasons: (1) QALD-8 and QALD-9 provide less training data; (2) QALD-8 and QALD-9 are more challenging as several questions are outside the scope of the system. Further analysis is discussed in the next subsection.

Table~\ref{tab:QALD8} and Table~\ref{tab:QALD9} report the question answering results on QALD-8 and QALD-9, respectively\footnote{The results of other methods are obtained from the result reports \cite{DBLP:conf/semweb/UsbeckNCRN18} and \cite{DBLP:conf/semweb/UsbeckGN018}.}. It is clear that the proposed method \textit{qaSQP} outperforms the state-of-the-art competitors greatly. 
 Basically, it benefits from the novel framework of answering questions. Specifically, the query graph is easy to retrieve by reducing the search space under the guidance of the recognized query sketch and one identified entity from the question. In contrast, the competitors are unaware of the query sketch, which will increase the difficulty in constructing the correct query graphs and retrieving the answers. For instance, the method \emph{qaSearch} performs much worse than \emph{qaSQP}, which confirms the superiority of \emph{SQP}-based framework.

\begin{table}
\centering
\begin{small}
\caption{Question answering results on QALD-8}
\begin{tabular}{lccc}
\toprule
Method  & \hspace{0.051in} Precision \hspace{0.051in} &\hspace{0.051in} Recall\hspace{0.051in} & F1-Measure \\
\midrule
QAKIS        & 0.061  & 0.053 &  0.056\\
WDAqua-core0 & 0.391  & 0.407 &  0.387     \\
gAnswer2     & 0.386  & 0.390 &  0.388     \\
qaSearch     & 0.244  & 0.244 &  0.244     \\
\textit{qaSQP}        & \textbf{0.459}  & \textbf{0.463} &  \textbf{0.461}      \\
\bottomrule
\end{tabular}
\label{tab:QALD8}
\end{small}
\end{table}

\begin{table}
\centering
\begin{small}
\caption{Question answering results on QALD-9}
\begin{tabular}{lccc}
\toprule
Method  & \hspace{0.051in} Precision \hspace{0.051in} &\hspace{0.051in} Recall\hspace{0.051in} & F1-Measure \\
\midrule
Elon         & 0.049  & 0.053 &  0.050 \\
QASystem     & 0.097  & 0.116 &  0.098 \\
TeBaQA       & 0.129  & 0.134 &  0.130 \\
WDAqua-core1 & 0.261  & 0.267 &  0.250     \\
gAnswer2     & 0.293  & 0.327 &  0.298     \\
qaSearch     & 0.236  & 0.241 &  0.237     \\
\textit{qaSQP}        & \textbf{0.458}  & \textbf{0.471} & \textbf{0.463}      \\
\bottomrule
\end{tabular}
\label{tab:QALD9}
\end{small}
\end{table}

\noindent \textbf{\emph{Evaluation of prediction models}}.
 Since a key component of the syetem is introducing the structural query patterns, the models that predict the \emph{SQP} are very important. So we study the performance of these prediction models. As presented in Table~\ref{tab:evaluation of prediction models}, the ensemble model outperforms the two individual models RNN-Attention and Text-CNN in terms of precision, recall, and F1 score on both QALD-8 and QALD-9. It means that the proposed ensemble model is effective.

\begin{table}
\vspace{-0.15in}
\centering
\begin{small}
\caption{Results of prediction models}
\begin{tabular}{lccc}
\toprule
Method (on dataset)  & Precision  & Recall & F1-Measure \\
\midrule
RNN-Attention (QALD-8)      & {0.82}  & 0.83 &  0.82     \\
Text-CNN (QALD-8)           & 0.79  & 0.78 &  0.78     \\
Ensemble model (QALD-8)     & \textbf{0.82}  & \textbf{0.86} &  \textbf{0.84}      \\
\hline
RNN-Attention (QALD-9)      & 0.83  & 0.78 &  0.80     \\
Text-CNN (QALD-9)           & 0.78  & 0.72 &  0.75     \\
Ensemble model (QALD-9)     & \textbf{0.85}  & \textbf{0.79 }&  \textbf{0.82}      \\
\bottomrule
\end{tabular}
\label{tab:evaluation of prediction models}
\end{small}
\end{table}

\noindent \textbf{\emph{Effect of SQP recognition and entity linking}}.
The modules of \emph{SQP recognition} and \emph{entity linking} are very critical in the proposed system. However, they are not guaranteed to produce the correct \emph{SQP} patterns or mapping entities. In order to study their effect and the boundary of the question answering ability, we conduct the experiments by providing the correct structural query patterns or mapping entities. Let \emph{qaSQP-CP} denote the method that is fed by the correct \emph{SQP}. Let \emph{qaSQP-CE} denote the method that is fed by one correctly identified entity initially.

As shown in Table~\ref{tab:effect of eqp and el}, all the methods equipped with correct \emph{SQPs} or entities outperform the original method \emph{qaSQP}. Note that the results on LC-QuAD are reported with respect to the performance on constructed structural query patterns. We observe that the improvement gained by \emph{qaSQP-CP} is subtle. Moreover, we can find that \emph{qaSQP-CE} performs much better than \emph{qaSQP-CP} on both QALD-8 and QALD-9. It indicates that the system \emph{qaSQP} can almost find the correct structural query patterns. Meanwhile, there is still much room to improve the initial entity linking.
\begin{table}
\centering
\begin{small}
\caption{Effect of \textit{SQP} recognition and entity lining}
\begin{tabular}{lccc}
\toprule
Method (on dataset) &  Precision  &  Recall  & F1-Measure \\
\midrule
\textit{qaSQP-CP} (LC-QuAD)     & 0.774  & 0.731 &  0.744     \\
\textit{qaSQP-CE} (LC-QuAD)     & 0.835  & 0.813 &  0.827     \\
\hline
\textit{qaSQP-CP} (QALD-8)     & 0.463  & 0.488 &  0.476     \\
\textit{qaSQP-CE} (QALD-8)     & 0.537  & 0.561 &  0.549     \\
\hline
\textit{qaSQP-CP} (QALD-9)     &  0.463  & 0.467 & 0.465      \\
\textit{qaSQP-CE} (QALD-9)    & {0.488}  & {0.502} & {0.493}     \\
\bottomrule
\end{tabular}
\label{tab:effect of eqp and el}
\end{small}
\end{table}

\noindent \textbf{\emph{Effect of the number of returned patterns $k$}}.
We also study the effect of the number of returned patterns, denoted by $k$, of the prediction model. Figures~\ref{fig:effect of k 8} and \ref{fig:effect of k 9} depict the results on QALD-8 and QALD-9, respectively. The parameter $k$ is varied from 1 to 3. As shown in the two figures, the precision, recall, and F1 score tend to be stable when $k$ is 2 and 3. Hence, $k$ is set to 2 by default in our experiments.

%

\begin{figure}[t]
  \centering
  \subfigure[Results on QALD-8]{
    \label{fig:effect of k 8} 
    \includegraphics[scale=0.3 ]{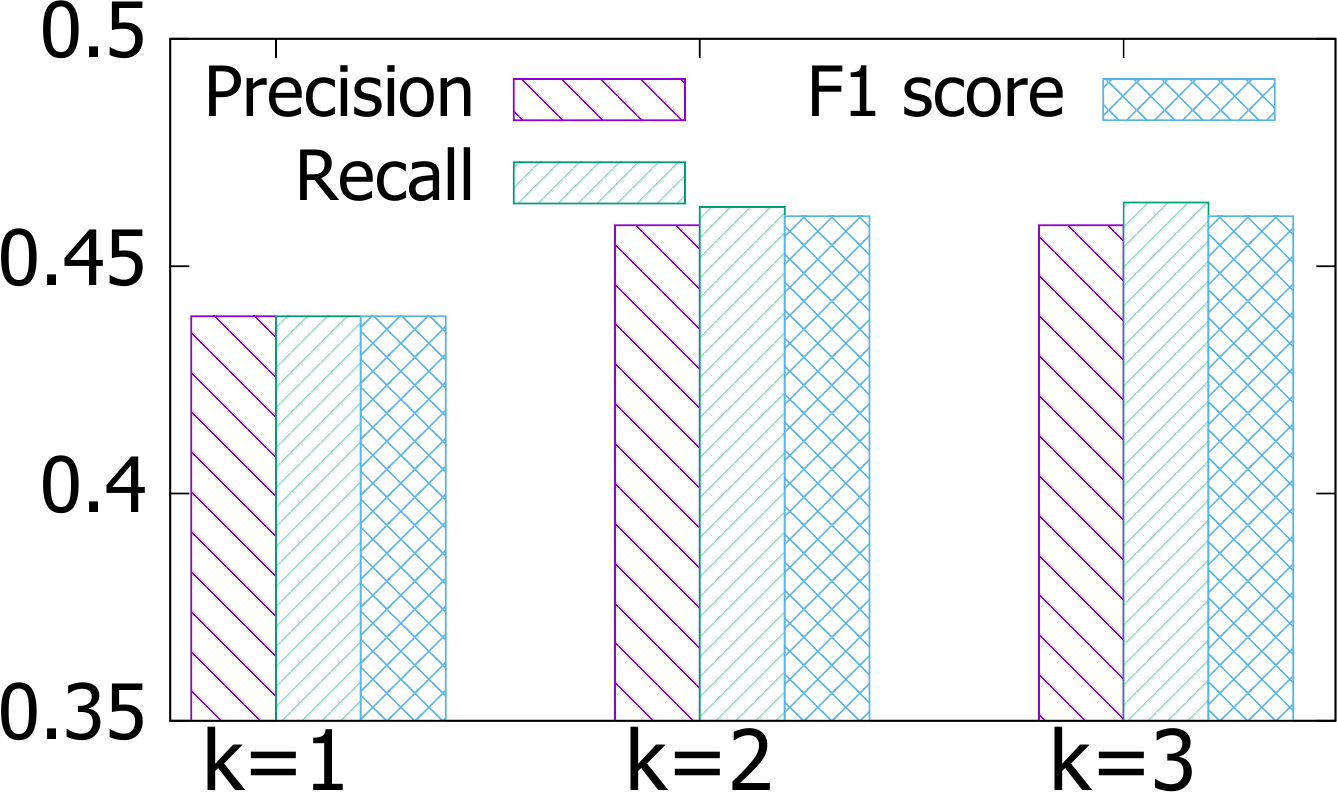}}
  \subfigure[Results on QALD-9]{
    \label{fig:effect of k 9} 
    \includegraphics[scale=0.3 ]{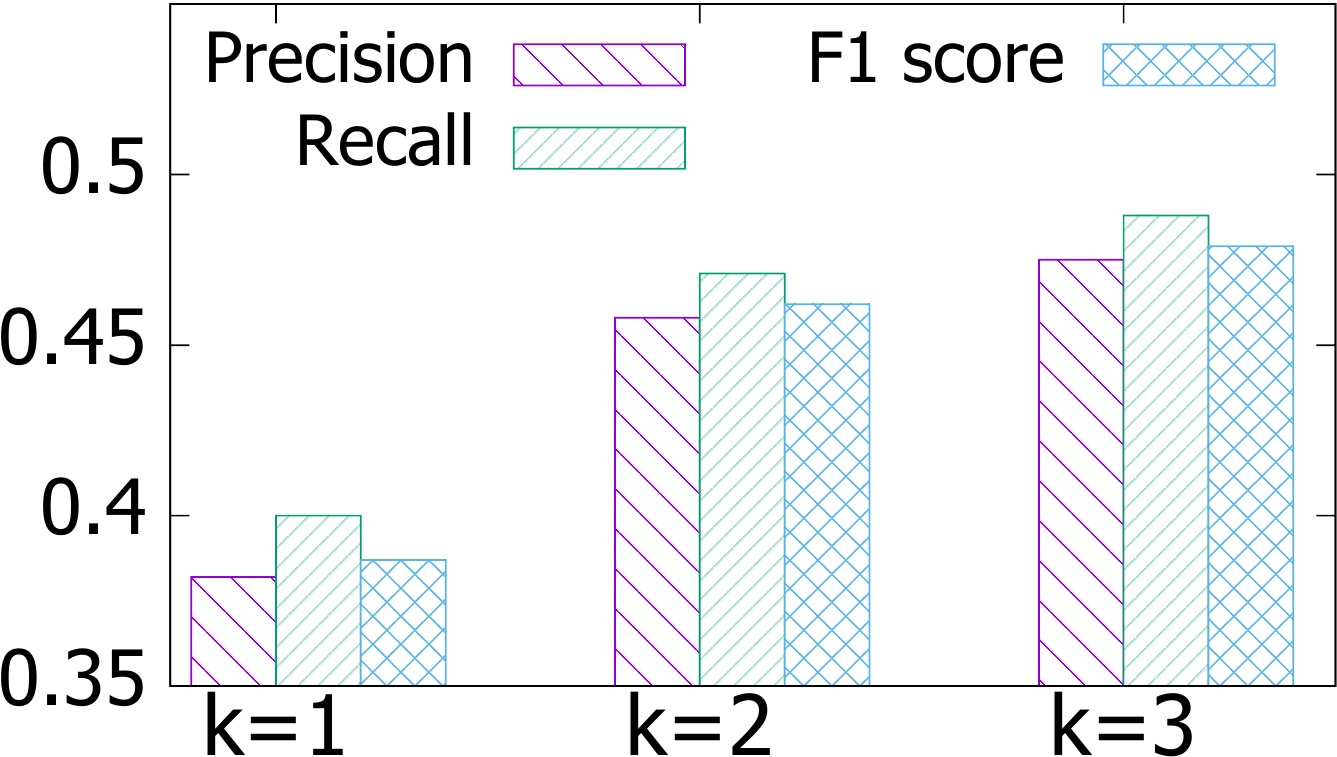}}
    \vspace{-0.15in}
  \caption{Effect of the number of returned patterns}
  \label{fig:effect of k} 
  \vspace{-0.05in}
\end{figure}

\noindent \textbf{\emph{Error Analysis}}.
Although our approach substantially outperforms existing methods, there is still much space for improving the performance on QALD-8 and QALD-9. For instance, besides the errors (29\%) caused by entity linking, the precision of predicted structural query patterns is 82\% for the test questions in QALD-8. Moreover, many questions in QALD-8 are very challenging. We find that 12 of 42 questions in QALD-8 leave out some important information or require external knowledge to find the correct answers, which increases the difficulty of answering for a system (41\%). For example, the question ``How big is the earth's diameter?'' cannot be answered directly since there is only a property ``meanRadius'' in DBpeida. To answer this question, the external knowledge that the diameter is two times the radius is required. The correct SPARQL query should be \textit{``select distinct (xsd:double(?radius)*2 AS ?diameter) where { res:Earth dbo:meanRadius ?radius. }''}. 17\% of the errors are caused by
incorrect label assignments in query graph extension.

\section{Conclusion and Future Work}\label{sec:conclusion}

In this paper, we focus on constructing query graphs for answering natural language questions over a knowledge graph. Unlike previous methods, we propose a novel framework based on structural query patterns. Specifically, we define structural query patterns that just capture the structural representations of input questions. Under the guidance of structural query patterns, the query graphs can be formulated. 
 Our experiments show that the proposed approach outperforms the competitors significantly in terms of building query graphs and generating answers. In the future, we will explore how to eliminate the effect of entity linking throughout the whole system. Applying structured learning techniques to SQP generation will also be investigated.


\bibliographystyle{aaai}
\bibliography{icde_sqp}

\end{document}